\documentclass[nohyperref]{article}

\usepackage{microtype}
\usepackage{graphicx}
\usepackage{booktabs} %

\usepackage{hyperref}

\usepackage[accepted]{icml2022}

\usepackage{amsmath}
\usepackage{amssymb}
\usepackage{mathtools}
\usepackage{amsthm}

\usepackage[capitalize,noabbrev]{cleveref}
\usepackage{subcaption}
\usepackage{caption}
\usepackage{pifont}
\usepackage{array}
\usepackage{multirow}
\usepackage{booktabs}
\usepackage{color, colortbl}
\usepackage{makecell}
\usepackage{float}
\usepackage{thm-restate}
\usepackage{diagbox}

\newcommand{\mbf}[1]{\mathbf{#1}}

\newcommand{\mbb}[1]{\mathbb{#1}}

\newcommand{\rvx}{\mathbf{x}}
\newcommand{\rvz}{\mathbf{z}}

\newtheorem{observation}{Observation}

\definecolor{Gray}{gray}{0.85}
\definecolor{LightCyan}{rgb}{0.88,1,1}

\makeatletter
\usepackage{xspace}
\def\@onedot{\ifx\@let@token.\else.\null\fi\xspace}
\DeclareRobustCommand\onedot{\futurelet\@let@token\@onedot}

\newcommand{\xmark}{\ding{55}}
\newcommand{\cmark}{\ding{51}}

\def\eg{\emph{e.g}\onedot}

\def\ie{\emph{i.e}\onedot}

\theoremstyle{plain}
\newtheorem{theorem}{Theorem}[section]
\newtheorem{proposition}[theorem]{Proposition}
\newtheorem{lemma}[theorem]{Lemma}

\theoremstyle{definition}
\newtheorem{definition}[theorem]{Definition}

\theoremstyle{remark}

\usepackage[textsize=tiny]{todonotes}

\icmltitlerunning{ButterflyFlow: Building Invertible Layers with Butterfly Matrices}

\begin{document}

\twocolumn[
\icmltitle{ButterflyFlow: Building Invertible Layers with Butterfly Matrices}

\icmlsetsymbol{equal}{*}

\begin{icmlauthorlist}
\icmlauthor{Chenlin Meng}{equal,s}
\icmlauthor{Linqi Zhou}{equal,s}
\icmlauthor{Kristy Choi}{equal,s}
\icmlauthor{Tri Dao}{s}
\icmlauthor{Stefano Ermon}{s}
\end{icmlauthorlist}

\icmlaffiliation{s}{Computer Science Department, Stanford University}

\icmlcorrespondingauthor{Chenlin Meng}{chenlin@cs.stanford.edu}
\icmlcorrespondingauthor{Linqi Zhou}{linqizhou@stanford.edu}
\icmlcorrespondingauthor{Kristy Choi}{kechoi@cs.stanford.edu}
\icmlcorrespondingauthor{Stefano Ermon}{ermon@cs.stanford.edu}

\icmlkeywords{Machine Learning, ICML}

\vskip 0.3in
]

\printAffiliationsAndNotice{\icmlEqualContribution} %

\begin{abstract}
Normalizing flows model complex probability distributions using maps obtained by composing invertible layers. 
Special linear layers such as masked and $1 \times 1$ convolutions play a key role in existing architectures because they increase expressive power while having tractable Jacobians and inverses.
We propose a new family of invertible linear layers based on butterfly layers, which are known to theoretically capture complex linear structures including permutations and periodicity, yet can be inverted efficiently. 
This representational power is a key advantage of our approach, as such structures are common in many real-world datasets.
Based on our invertible butterfly layers, we construct a new class of normalizing flow models called ButterflyFlow.
Empirically, we demonstrate that ButterflyFlows not only achieve strong density estimation results on natural images such as MNIST, CIFAR-10, and ImageNet-32$\times$32, but also obtain significantly better log-likelihoods on structured datasets such as galaxy images and MIMIC-III patient cohorts---all while being more efficient in terms of memory and computation
than relevant baselines.

\end{abstract}

\section{Introduction}
Generative models have  achieved tremendous success in a wide range of domains, such as images \citep{brock2018large,karras2020analyzing,vahdat2020nvae,ho2020denoising,song2020score}, natural language \citep{brown2020language,chowdhery2022palm}, video \citep{kumar2019videoflow, ho2022video}, molecule synthesis \citep{kadurin2017drugan,de2018molgan,gomez2018automatic}, and speech \citep{oord2018parallel,kong2020diffwave}. 
Normalizing flows, in particular, have attracted significant attention since they allow \emph{exact} likelihood evaluation of data rather than lower-bound approximations \citep{dinh2014nice,kingma2018glow}.

To build such normalizing flows, one must design flexible families of functions that are both invertible and admit efficient computation of Jacobian determinants \citep{rezende2015variational,papamakarios2019normalizing,hoogeboom2020convolution,NEURIPS2019_b1f62fa9,finzi2019invertible,hoogeboom2019emerging,chen2019residual,ho2019flow++,grcic2021densely}. 
While the development of non-linear coupling layers fueled early progress in the field \citep{dinh2014nice,dinh2016density}, recent advances have focused on the 
effectiveness of special linear layers such as masked, $1 \times 1$, and $d \times d$ convolutions
as key architectural primitives, among others~\cite{ma2019macow,kingma2018glow,hoogeboom2019emerging,hoogeboom2020convolution}. 
In particular, most state-of-the-art flow models first preprocess the data with such linear layers while also leveraging non-linear layers for expressivity.

\looseness=-1
In this work, we draw inspiration from the literature on learning efficient, structured linear transformations
and propose a new class of invertible linear layers based on 
\emph{butterfly layers} \citep{dao2019learning}. 
Our invertible butterfly layer satisfies the usual desiderata of a normalizing flow primitive. 
However, its key distinguishing feature lies in its representational power: in spite of its efficiency, it inherits desirable properties from \citealt{dao2019learning} in that it is theoretically guaranteed to capture complex structures in data such as permutations and periodicity. 
The expressivity of invertible butterfly layers gives it an advantage over existing methods when \emph{modeling real-world datasets that exhibit such structures}.
We then construct a new family of normalizing flow models called ButterflyFlow by combining our proposed invertible butterfly layers with coupling layers \cite{dinh2016density} and a Glow-based model backbone \cite{kingma2018glow}.

Empirically, we demonstrate that ButterflyFlow is an effective  generative model, performing favorably relative to existing methods on image datasets such as MNIST, CIFAR-10, and ImageNet-32$\times$32. %
However, we highlight that ButterflyFlow shines when modeling real-world data with special underlying structures, such as periodicity and permutations. 
Our model outperforms relevant baselines on the MIMIC-III patient dataset by approximately 200\% in negative log-likelihoods per  dimension while requiring \emph{less than half the number of model parameters}.
In this way, our invertible butterfly layer serves as a powerful architectural primitive for capturing global regularities present in the data.

The contributions of our work can be summarized as:
\begin{enumerate}
    \item We introduce ButterflyFlow, a new class of flow-based generative models parameterized by butterfly matrices.
    \item We provide theoretical guarantees that ButterflyFlow can efficiently capture common types of structures, such as permutations.
    \item We show empirically that ButterflyFlow achieves strong performance on density estimation and image synthesis tasks, and is superior at modeling data with special structure (e.g. periodicity) in real-world settings relative to existing flow-based models. 
\end{enumerate}

\section{Preliminaries}
\label{prelim}
\subsection{Flow-based Generative Models}
Given a data distribution $p_X(\mbf{x})$ and a base distribution $p_Z(\mbf{z})$ (\eg, a Gaussian distribution),
a normalizing flow is an invertible transformation $f_{\theta}: \mbf{x} \in \mbb{R}^n \mapsto \mbf{z} \in \mbb{R}^n$ that 
approximates $p_X(\mbf{x})$ via the change of variables formula:
\begin{equation}
\label{eq:flow_objective}
    p_{\theta}(\mbf{x}) = p_Z(\mbf{z})|\operatorname{det} J_{f_{\theta}}(f_{\theta}^{-1}(\mbf{z}))|,
\end{equation}

where $J_{f_{\theta}}$ is the Jacobian of $f(\mbf{x})$, and $\theta$ is the set of learnable parameters.
In practice, the Jacobian determinant $\operatorname{det} J_{f_{\theta}}(f_{\theta}^{-1}(\mbf{z}))$ must be tractable to compute.
Coupled with a simple $p_Z(\rvz)$, the change of variables formula allows for the exact likelihood evaluation of a complex $p_X(\rvx)$ as well as maximum likelihood training of $f_{\theta}$.
To sample a new data point from the model, we first draw samples $\rvz \sim p_Z(\rvz)$ from the prior distribution and then push it through the inverse flow transformation: $\rvx = f_{\theta}^{-1}(\rvz)$.

Because the normalizing flow's ability to capture complex $p_X(\rvx)$ hinges on the expressivity of the transformation $f_{\theta}$,
recent works have focused on developing more flexible parameterizations of $f_{\theta}$. 
In particular, both non-linear and linear layers have demonstrated promise.
\paragraph{Non-linear coupling layers.}

Coupling layers~\cite{dinh2014nice,dinh2016density} are a powerful class of invertible non-linear layers. 
The coupling layer splits the input $\mbf{x}$ into two components: $\mbf{x}_a$ and $\mbf{x}_b$. 
Then, it applies an identity map to $\mbf{x}_a$ and transforms $\mbf{x}_b$ using a learnable affine transform (with shift and scale parameters $s_{\theta}$ and $b_{\theta}$) that depend on $\mbf{x}_a$. The output of this layer $\mbf{y}$ is obtained by concatenating these two intermediate quantities:
\begin{align}
    &\mbf{z}_a = \mbf{x}_a; \: \mbf{z}_b = \mbf{x}_b\odot s_{\theta}(\mbf{x}_a) + b_{\theta}(\mbf{x}_a) \nonumber \\
    &\mbf{y}  = \text{concat}(\mbf{z}_a, \mbf{z}_b) \nonumber
\end{align}
Due to its simplicity and efficiency, the coupling layer has become a fundamental building block for most state-of-the-art flow model architectures ~\cite{chen2020vflow,ma2019macow,ho2019flow++}. 
However, their effectiveness depends heavily on the way in which the input $\rvx$ is partitioned.
Recent works have shown that \emph{linear layers} can learn an improved partitioning scheme, thereby boosting the performance of downstream coupling layers when used together~\cite{kingma2018glow}.

\paragraph{Invertible linear layers.}
Linear layers, such as invertible $1 \times 1$ convolutions,
were designed to increase the effectiveness of coupling layers when paired together.
Specifically, they \emph{learn} a more general partitioning of the input than the naive splitting as done in conventional coupling layers~\cite{kingma2018glow}.
Given an input with channel size $c$, we denote the learnable parameter (\ie, the filter of the $1 \times 1$ convolution) as $\textbf{W}\in\mathbb{R}^{c\times c}$. 
To compute the Jacobian determinant efficiently,
\citeauthor{kingma2018glow} use LU decomposition and parameterize $\textbf{W}$ as:
\begin{equation}
    \textbf{W}=\textbf{P}\textbf{L}(\textbf{U}+\text{diag}(\textbf{s})),
\end{equation}
where $\textbf{P}$ is a pre-specified orthogonal matrix, $\textbf{L}$ is a lower triangular matrix with ones on the diagonal, $\textbf{U}$ is an upper triangular matrix with zeros on the diagonal, and $\textbf{s}$ is a $c$-dimensional vector~\cite{kingma2018glow}. 
This particular structure in the matrix decomposition allows for the Jacobian determinant to be computed in $\mathcal{O}(c)$, rather than $\mathcal{O}(c^3)$. 
Other invertible linear layers, such as the Emerging convolution and the Woodbury transformation~\cite{hoogeboom2019emerging,lu2020woodbury}, leverage similar types of matrix structures such as sparsity to improve the performance of coupling layers without sacrificing efficiency.

\subsection{Butterfly Layers for Efficient Structured Transforms}
The butterfly layer is a special family of linear layers that can be represented as a product of $K$ sparse matrices called \emph{butterfly factors}~\cite{parker1995random,dao2019learning,dao2020kaleidoscope}. 
The butterfly factor has a particular structure that requires the specification of two parameters: the level $i \in [K]$ 
and the factor dimension $D$. We assume that $D$ is a power of 2 for ease of the technical exposition.

\textbf{Level-one butterfly factor.}
A level-one $D$-dimensional butterfly factor $\textbf{B}(1,D)$ is a $D\times D$ sparse matrix.
Its only non-zero entries are the diagonals of the four $D/2\times D/2$ sub-matrices obtained by partitioning the matrix in half~\cite{dao2019learning}, as shown in the left panel in \cref{fig:butter_block}.

\textbf{Level-$i$ butterfly factor.}
More generally, a level-$i$ $D$-dimensional butterfly factor $\textbf{B}(i,D)$
is a $D\times D$ block diagonal matrix (\ie, any off-diagonal block is a zero matrix) with block size $D/2^{i-1}\times D/2^{i-1}$.
Each of the diagonal blocks is a $(D/2^{i-1})$-dimensional level-one butterfly factor.  
Therefore, a level-$i$ $D$-dimensional butterfly factor has the form:
\begin{equation*}
\resizebox{\linewidth}{!}{
   $ \textbf{B}(i,D)=
    \begin{bmatrix}
    \textbf {B}_{1}(1, D/2^{i-1}), &\textbf{0} &... &\textbf{0}\\
    \textbf{0} &\textbf {B}_{2}(1, D/2^{i-1}) &... &\textbf{0}\\
    &...\\
    \textbf{0} &\textbf{0} &... &\textbf{B}_{2^{i-1}}(1, D/2^{i-1})
    \end{bmatrix}$
    }
\end{equation*}
where $\textbf{B}_{j}(1, D/2^{i-1})$ is a level-one $D/2^{i-1}$-dimensional butterfly factor on the $j$th sub-block of $\textbf{B}(i, D)$,
and $\textbf{0}$ is the $D/2^{i-1}\times D/2^{i-1}$ zero matrix.  We provide an illustrative example for $D=16$ in \cref{fig:butter_block}.

\begin{figure}[H]
    \centering
    \includegraphics[width=\linewidth]{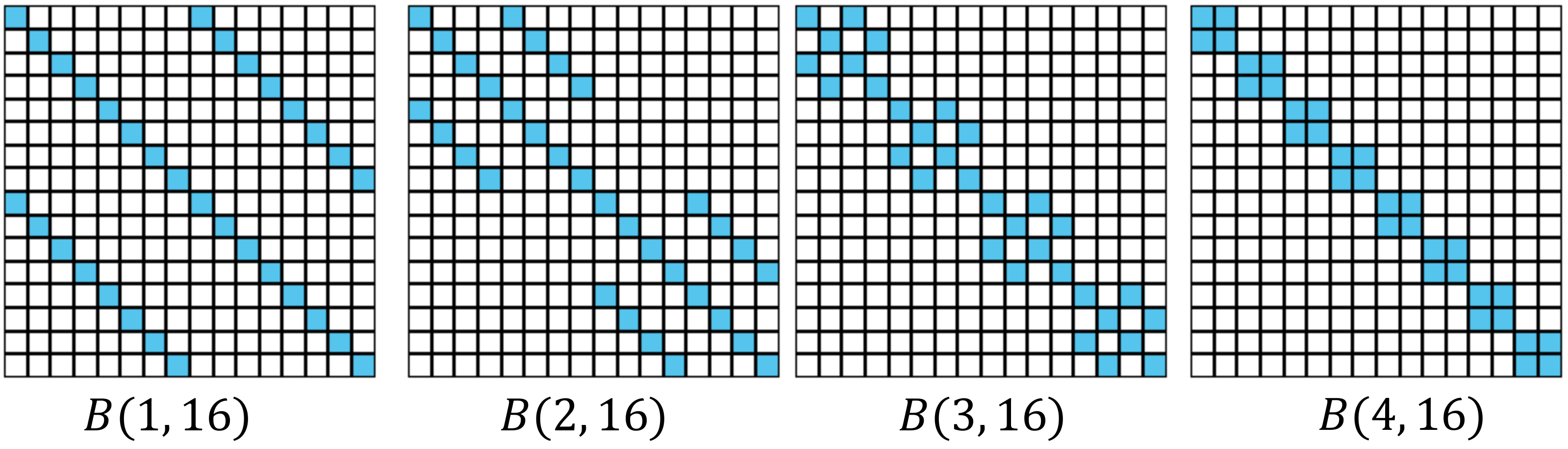}
    \caption{Butterfly factors $B(i,D)$ ($D=16$) at level $i=1, 2, 3$, and 4. White entries denote zeros and blue entries represent non-zero values. Each butterfly factor is parameterized by the non-zero values in the blue entries. 
    }
    \vspace{-5mm}
    \label{fig:butter_block}
\end{figure}

We can now construct a \emph{butterfly layer} by composing a sequence of butterfly factors as defined below.
\begin{definition} [Butterfly layer]
Given a $D$-dimensional input $\mbf{x}$, a \emph{$D$-dimensional butterfly layer} is 
 a linear layer $b:\mbf{x}\to \textbf{B}(D)\mbf{x}$, where $\textbf{B}(D)=\textbf{B}(a_1, D) \textbf{B}(a_2, D)...\textbf{B}(a_k, D)$  is a product of butterfly factors and $\{a_i\}^{k}_{i=1}$ is a sequence of integers 
such that $D \equiv 0 \pmod{2^{a_i}}$.
\end{definition}
As an example, a commonly used butterfly layer in the literature (called the butterfly matrix)~\cite{parker1995random, dao2019learning,dao2020kaleidoscope} can be constructed as $\textbf{B}(D)=\textbf{B}(1, D) \textbf{B}(2, D)...\textbf{B}(k, D)$, 
where $k$ is the largest integer 
such that $ D \equiv 0 \pmod{2^{k}}$.

\subsection{Desirable Properties of Butterfly Layers}
Although the construction of a butterfly layer involves a sequence of matrix multiplications, it can be efficiently computed on modern hardware by utilizing the sparsity of butterfly factors~\cite{dao2019learning,dao2020kaleidoscope}. 
The butterfly layers' efficiency belies their expressivity: they can represent a wide variety of structured linear maps including discrete Fourier transforms, permutations, and convolutions~\cite{dao2020kaleidoscope,dao2019learning}.
In the following section, we port over such advantages into generative models by first proposing a new family of \emph{invertible} butterfly linear layers, which serve as a useful architectural primitive for flow models.

\section{Building Invertible Butterfly Layers}
Recall that each transformation in a normalizing flow layer must be invertible and have an efficiently-computable Jacobian determinant.
We describe how flow layers comprised of butterfly factors satisfy both desiderata.

\subsection{Invertible Butterfly Factors for Normalizing Flows
}
\label{sec:Invertible Butterfly Factors}
\looseness=-1
We first demonstrate how to compute the Jacobian determinant of the butterfly factor.
Given an input $\mbf{x} \in \mbb{R}^{D}$,
we denote the parametrized level-$i$ butterfly factor as $\textbf{B}_{\theta}{(i,D)}$, where $\theta$ is the set of learnable parameters (values of the non-zero entries corresponding to the blue entries in \cref{fig:butter_block}).
Then given the linear transformation $b_i:\mbb{R}^{D}\to \mbb{R}^{D}: \mbf{x} \to \textbf{B}_{\theta}{(i,D)} \mbf{x}$, 
we can see that the Jacobian $J_{b_i}(\mbf{x})= \textbf{B}_{\theta}{(i,D)}$.
Thus, computing the Jacobian determinant of the mapping $b_i$ is equivalent to computing the Jacobian determinant of the butterfly factor $\textbf{B}_{\theta}{(i,D)}$, which can be done efficiently as in \cref{thm:jacobian}.
\begin{restatable}[]{theorem}{determinant}
\label{thm:jacobian}
The determinant of any $D$-dimensional butterfly factor can be computed in $\mathcal{O}(D)$.
\end{restatable}
\begin{proof}[Proof sketch]
We provide the full proof in Appendix~\ref{sec:proof}. 
Since we can decompose the matrix into diagonal matrices, computing the determinant only involves operations on the diagonal elements.
\end{proof}
Next, we consider the invertibility of the butterfly layer.
When $\textbf{B}_{\theta}{(i,D)}$ is non-singular, the transformation $b_i(\mbf{x})$ is invertible. 
More formally, we define an invertible butterfly factor as the following:
\begin{definition}[Invertible butterfly factor]
 An \emph{invertible} level-$i$ $D$-dimensional butterfly factor $\textbf{B}_{\theta}{(i,D)}$ is a $D\times D$ \emph{non-singular} level-$i$ $D$-dimensional butterfly factor.
\end{definition}

Additionally, we can see that the inverse transformation of $b_i$ is $b_i^{-1}:\mbb{R}^{D}\to \mbb{R}^{D}: \mbf{x} \to \textbf{B}^{-1}_{\theta}{(i,D)} \mbf{x}$, where $\textbf{B}^{-1}_{\theta}{(i,D)}$ is the matrix inverse of $\textbf{B}_{\theta}{(i,D)}$. 
Thus computing $\mbf{x}=b_i^{-1}(\mbf{z})$ only requires the application of the following linear transformation to $\rvz$: 
\begin{equation}
    \label{eq:inverse}
    \mbf{x}=b_i^{-1}(\mbf{z})=\textbf{B}^{-1}_{\theta}{(i,D)}\mbf{z}.
\end{equation}
We note that although \cref{eq:inverse} involves a potentially expensive matrix multiplication of a $D\times D$ matrix inverse with a $D$-dimensional vector,
we can efficiently invert $\textbf{B}_{\theta}(i,\theta)$ given the following proposition.
\begin{restatable}[]{proposition}{inverse}
\label{pro:inverse}
Assuming $\textbf{B}_{\theta}{(i,D)}$ is non-singular, the matrix $\textbf{B}_{\theta}^{-1}{(i,D)}$ is a $D$-dimensional level-$i$ butterfly factor that can be computed in $\mathcal{O}(D)$. Given $\textbf{B}_{\theta}^{-1}{(i,D)}$,
 the map $b_i^{-1}: \mbf{z}\to \textbf{B}_{\theta}^{-1}{(i,D)}\mbf{z}$ can be computed in $\mathcal{O}(D)$.
\end{restatable}
We provide the proof in \cref{sec:proof}.
\cref{pro:inverse} together with \cref{thm:jacobian} show that butterfly factors can be made efficiently invertible with tractable Jacobian determinants, making them suitable as building blocks for flow-based generative models.

\subsection{Invertible Butterfly Layers}
\label{sec:invertible_butter_layer}
With our invertible butterfly factors in place, we introduce a way to compose them into a more powerful \emph{invertible butterfly layer}. 
We make this precise in the following definition.
\begin{definition}[Invertible butterfly layer]
An \emph{invertible} butterfly layer $b$ is defined as
\begin{equation}
\label{eq:butter_layer}
    b=b_{a_1}\circ b_{a_2}\circ ... \circ b_{a_k}, 
\end{equation}
where  $b_{a_i}:\mbf{x}\to \textbf{B}_{\theta}(a_i,D)\mbf{x}$ are  \emph{invertible} butterfly factors and $\{a_i\}^{k}_{i=1}$ is a sequence of integers such that 
$D \equiv 0 \pmod{2^{a_i}}$.
\label{def:inv_butter_layer}
\end{definition}

Definition~\ref{def:inv_butter_layer} suggests that by virtue of being a composition of invertible butterfly factors, the invertible butterfly layer $b$ inherits some of their nice properties.
Specifically, let us consider the Jacobian determinant of $b$
in \cref{eq:butter_layer}.
Using the chain rule:
\begin{equation}
    \log |\text{det}J_{b}(\mbf{x})|=\sum_{i=1}^{k}\log|\text{det}J_{b_{a_i}}(\mbf{x})|.
\end{equation}
Since each invertible butterfly factor $b_{a_i}$ can be efficiently inverted with a Jacobian determinant that can be computed in $\mathcal{O}(D)$, their composition $b$ is also efficiently invertible with a Jacobian determinant that can be computed in $\mathcal{O}(kD)$. 

\looseness=-1
In addition to their efficiency and ease of invertibility, invertible butterfly layers largely
retain the expressiveness of the original butterfly layers \citep{dao2019learning,dao2020kaleidoscope}.
As a concrete example, 
they can represent any permutation matrix. 

\begin{restatable}[]{proposition}{permutation}
\label{pro:permutation}
Any  $D\times D$ permutation matrix (with $D=2^{k}$ a power of 2) can be represented by an invertible butterfly layer.
\end{restatable}
The proof of \cref{pro:permutation} follows \citeauthor{dao2020kaleidoscope}, and we provide more details in \cref{sec:proof}. 
\cref{pro:permutation} shows that invertible butterfly layers can also act as \emph{learnable} permutation layers.
This is especially helpful for adding expressivity when our butterfly layers are paired with nonlinear coupling layers that use a fixed partitioning of the input.

\subsection{Block-wise Invertible Butterfly Layers}
\label{channel_wise}

We also introduce a new variant of our invertible butterfly factor called the block-wise butterfly factor.
Specifically, given a $D$-dimensional input $\mbf{x}$, we partition its entries into $D/C$ groups where each group has $C$ elements (see \cref{fig:3d_butter_block}). 
We assume that $C$ divides $D$ for simplicity.

\begin{definition}[Block-wise invertible butterfly factor]
A level-$i$, block-size-$C$, $D$-dimensional \emph{block-wise invertible butterfly factor} $\textbf{B}_{\theta}(i,D,C)$ is a $D\times D$ non-singular block matrix with block size $C\times C$ 
such that
for any $j,\hat{j} \in\{1,...,C\}$, 
the $D/C \times D/C$ sub-matrix of $\textbf{B}_{\theta}(i,D,C)$ obtained by selecting the $C\cdot l+j$-th rows and the $C\cdot l+\hat j$-th columns  for $l \in \{0,...,D/C-1\}$,
is a level-$i$ $D/C$-dimensional butterfly factor. 
\end{definition}
\looseness=-1
Intuitively, the block-wise butterfly factor is a $D\times D$ block matrix whose $C\times C$ blocks satisfy the sparsity pattern of a $\textbf{B}_{\theta}(i,D/C)$ butterfly layer.
We provide an illustrative example in \cref{fig:channel_butter_block}.
Unlike the naïve butterfly factor where only two entries per row are allowed to be non-zero (see \cref{fig:butter_block}), this modification allows for at most $2C$ non-zero entries per row.

\begin{figure}[h]
    \centering
    \includegraphics[width=\linewidth]{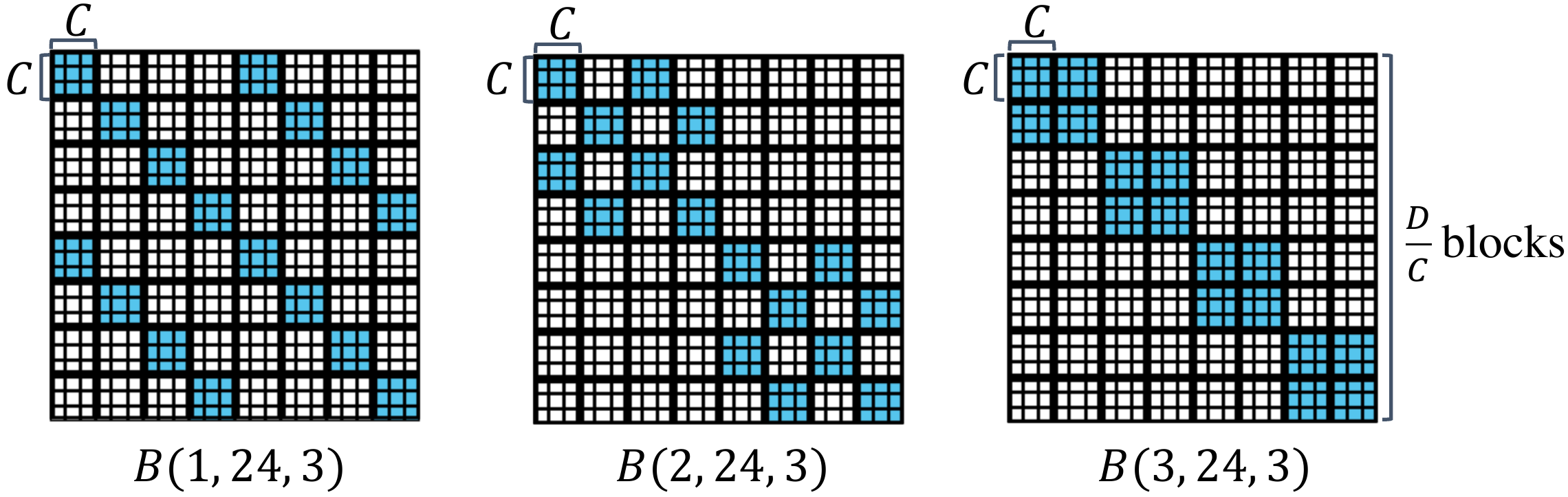}
    \caption{Block-wise invertible butterfly factors $B(i,D,C)$ ($D=24, C=3$) at levels $i=1, 2, 3$. White entries
    denote zeros and blue entries denote non-zeros. A block-wise invertible butterfly factor can have $2C$ non-zero entries per row and is more expressive than the naïve butterfly factor (\cref{sec:invertible_butter_layer}), which can only have $2$ non-zero entries per row.
    }
    \label{fig:channel_butter_block}
\end{figure}

\paragraph{Constructing block-wise invertible layers.}
Similar to how an invertible butterfly layer is constructed using invertible butterfly factors, a block-wise invertible butterfly \emph{layer} is constructed by composing a series of block-wise invertible butterfly \emph{factors}.
The block-wise invertible butterfly layer not only improves the flexibility of our invertible butterfly layers, but also reveals interesting connections to previous methods as in the following observations. 

\begin{observation}
\label{observation:butterfly}
When $C=1$, the block-wise invertible butterfly layer reduces to the invertible butterfly layer discussed in \cref{sec:Invertible Butterfly Factors}. 
\end{observation}
In fact, the block-wise butterfly layer generalizes commonly used invertible linear layers such as the 1x1 convolution~\cite{kingma2018glow}. 

\begin{observation}
\label{pro:1x1}
When $C$ is set to the input's channel size, the block-wise invertible butterfly layer recovers the invertible 1x1 convolution by setting non-diagonal blocks to be zero and using tied weights for diagonal blocks. This is the byproduct of grouping the input entries by channels. 
\end{observation}
Additionally, the following observation shows that allowing the weights of the block-wise invertible butterfly layers to be complex numbers confers a significant boost to their representational power.

\begin{restatable}[]{observation}{observationconv}
\label{pro:emerging}
The block-wise invertible butterfly layer with weights in $\mathbb{C}$ can be used to represent a subset of the invertible $d\times d$ convolution layer.
\end{restatable}
Specifically, butterfly layers with weights in $\mathbb{C}$ can express any $d\times d$ convolution that can be decomposed into a channel-wise mixing (\eg channel-wise matrix multiplication) and a channel-wise convolution (\ie spatial convolution for each channel). This is an extension of a property of complex-valued naïve butterfly matrices, which can represent any 1D periodic convolution \citep{dao2019learning}. We provide additional discussion on this point in \cref{sec:proof}.
Although butterfly layers can have weights in both $\mathbb{C}$ and $\mathbb{R}$, we 
empirically observe that restricting the butterfly layer weights to be in $\mathbb{R}$ yields good performance, and only consider real-valued weights in the rest of the paper.

\paragraph{Computational complexity.}
\looseness=-1
There exists a trade-off between flexibility and computational complexity in block-wise butterfly layers---larger values of $C$ correspond to more powerful but (potentially) more computationally expensive models.
To address this, we use a more efficient parameterization of the block-wise butterfly factor: each $C \times C$ block is implemented with LU decomposition~\cite{kingma2018glow}, 
which reduces the complexity of computing each of the $D/C$ Jacobian determinants of the $C\times C$ block from $\mathcal{O}(C^3)$ to $\mathcal{O}(C)$. 
Then, since the Jacobian determinant of a naïve butterfly layer can be computed in $\mathcal{O}(D)$, the Jacobian determinant of the block-wise butterfly layer can be evaluated in $\mathcal{O}(D)$. 
Similarly, with LU decomposition for each $C\times C$ block,%
we can show that the inverse of the block-wise invertible butterfly layer with $k$ block-wise butterfly factors can be computed in $\mathcal{O}(kC^2D)$.
This is because the desired computation reduces to inverting a sequence of block-wise butterfly factors. 

\begin{figure}[h]
    \centering
    \includegraphics[width=\linewidth]{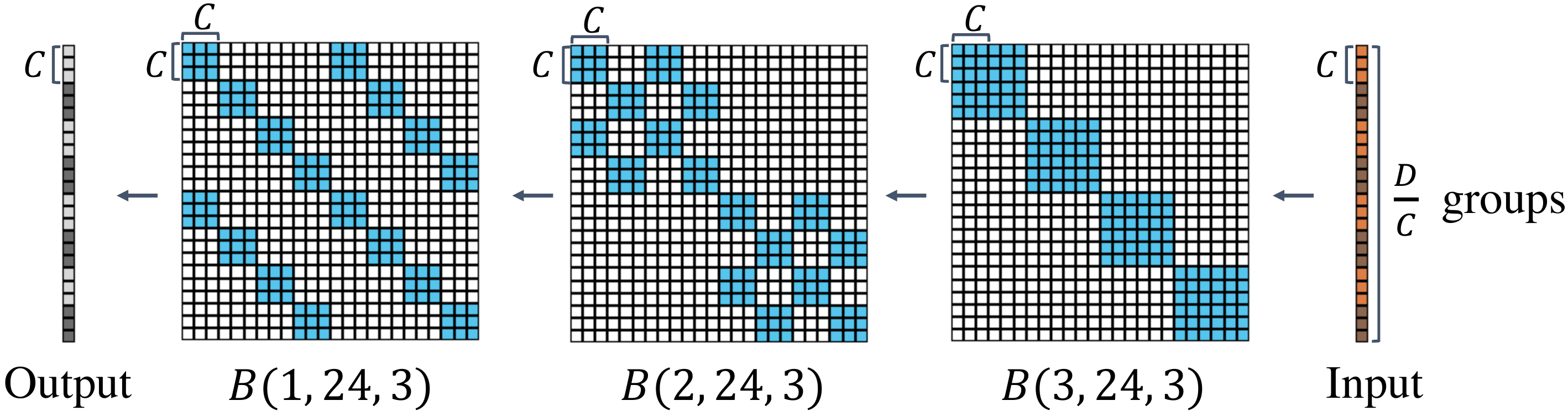}
    \caption{
    An example of a block-wise invertible butterfly \emph{layer} ($D=24, C=3$), which is constructed by composing a sequence of block-wise invertible butterfly \emph{factors}. White entries denote zeros and blue entries denote non-zero values. A $24$-dimensional input is partitioned into $8$ groups where each group has block-size $C=3$ before feeding into the block-wise butterfly layer.
    }
    \label{fig:3d_butter_block}
    \vspace{-3mm}
\end{figure}

\section{Generative Modeling with ButterflyFlow}
\subsection{Architectural Components}
In this section, we introduce how to construct the ButterflyFlow model by leveraging our (block-wise) invertible butterfly layers from Section~\ref{channel_wise}. 
We consider the following invertible layers as architectural building blocks that will be combined together for the final model.

\textbf{Coupling layers.} 
As discussed in Section~\ref{prelim}, the coupling layer~\cite{dinh2014nice,dinh2016density} is a standard primitive in most state-of-the-art normalizing flow models. 
We similarly leverage such coupling layers to increase the expressivity of our ButterflyFlow model.

\textbf{Split and squeeze layers.} \citeauthor{dinh2016density} split and reshuffle the input dimensions for better mixing.
This allows for constructing deeper stacks of coupling layers within the same flow model, increasing its expressive power.
We use them in combination with the above mentioned coupling layers to improve their performance, as done in prior works.

\textbf{Actnorm layers.} 
Actnorm layers are invertible normalization layers that have been developed as an alternative to batch normalization \citep{ioffe2015batch} in flow-based generative models \cite{kingma2018glow}. 
Their parameters are initialized in a data-dependent way \citep{hoogeboom2019emerging,ma2019macow}.
They linearly transform the activations of the input using a scale and translation parameter similar to affine coupling layers, and have been shown to improve training stability.

\looseness=-1
\textbf{Invertible Butterfly layers.}
Given an input $\rvx$, we expand it into a $D$-dimensional vector before feeding it into the block-wise butterfly layer as in Figure~\ref{fig:3d_butter_block}. 
Each layer's block-size $C$ and grouping mechanism are specific to each particular data type.
In the case of RGB images, the block-wise butterfly layers use $C=3$ and group together RGB values of the same pixels (\ie, cells of the same colors in the input vector shown in \cref{fig:3d_butter_block}).

\subsection{Building the ButterflyFlow Model} 
\begin{figure}
\centering
\includegraphics[width=0.73\linewidth]{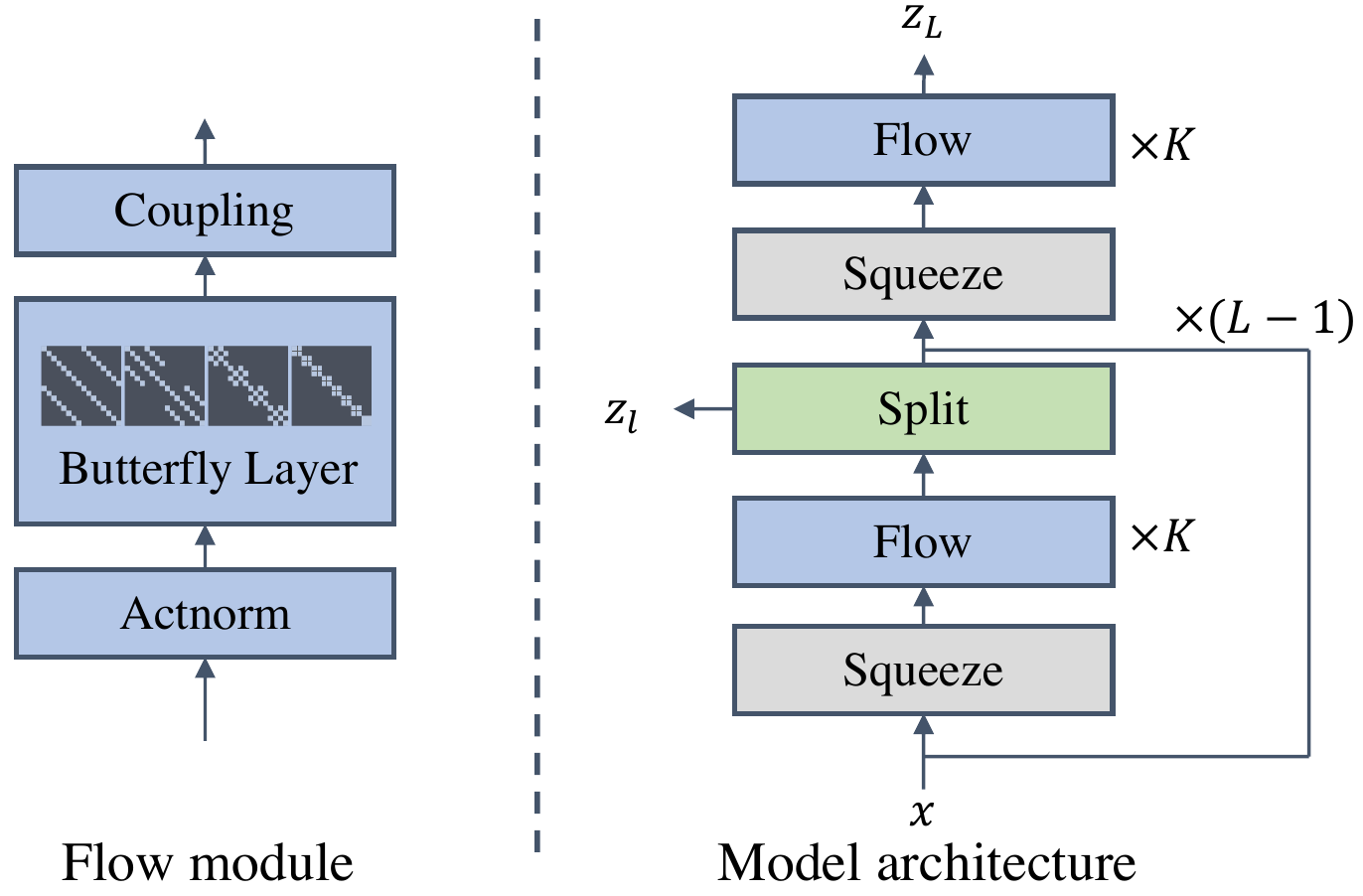}
\caption{
Architecture overview of ButterflyFlow, which shows the transformation of an input $\rvx$ to an intermediate output $\rvz_L$. The left diagram details a single Flow module on the right: the input is first passed through an Actnorm layer~\cite{kingma2018glow} then transformed by a block-wise invertible butterfly layer,
whose output is then fed into a coupling layer~\cite{dinh2016density}. 
Each Flow module is repeated $K$ times. Similar to \cite{kingma2018glow}, we use a hierarchical prior using the split module. We also use the squeeze layer as in \citep{dinh2016density} to split and reshuffle dimensions for better mixing. 
}
\label{fig:butterfly_flow}
\end{figure}
Following recent architectural advancements in flow-based models~\cite{hoogeboom2019emerging,ma2019macow}, ButterflyFlow stacks a series of squeeze, Flow, and split modules together.
This results in an architecture of $L$ levels and $K$ Flow modules per level as shown in Figure~\ref{fig:butterfly_flow}. %
Within each Flow module, we combine our invertible butterfly layers with Actnorm layers and Coupling layers for added expressivity~\cite{hoogeboom2019emerging}.
We elaborate upon our design decisions as well as hyperparameter recommendations for ButterflyFlow
in \cref{app:model}.

Maximum likelihood training of the ButterflyFlow model proceeds in the same fashion as in conventional flow-based generative models via~\cref{eq:flow_objective}.
While maintaining the invertibility of the (block-wise) invertible butterfly layers $\textbf{B}_{\theta}{(i,D, C)}$ during training may present a concern, we note that we do not need to enforce any additional constraints to ensure that $\textbf{B}_{\theta}{(i,D, C)}$ remains non-singular.
This is because the training loss will become infinitely large when $\text{det} (\textbf{B}_{\theta}{(i,D, C)})=0$ (see \cref{eq:flow_objective}). 
In particular, for a butterfly layer, a local non-zero Jacobian determinant (e.g. evaluated at a particular data point) implies a non-zero Jacobian determinant globally---this means that the layer will be invertible.
This special property of the butterfly layer is not generally applicable to conventional model architectures.

\section{Experiments}
In this section, we are interested in investigating three broad questions empirically:
\begin{enumerate}
    \item How effective is ButterflyFlow at density estimation tasks on standard natural image datasets?
    \item How well can ButterflyFlow model datasets with special structures, such as permutation and periodicity?
    \item Is ButterflyFlow indeed more efficient than relevant baselines in terms of wall-clock time and/or memory?
\end{enumerate}
We evaluate ButterflyFlow on both synthetic and real datasets that have the corresponding structures of interest.
We provide additional details on specific experimental settings and hyperparameter configurations in Appendix~\ref{app:exp_details}.

\subsection{Density estimation on images}
We first benchmark our method on standard image datasets to ensure that ButterflyFlow still performs favorably on the usual generative modeling tasks.

\textbf{Datasets.} As in prior works~\cite{hoogeboom2019emerging, lu2020woodbury}, we evaluate our method's performance on MNIST~\cite{deng2012mnist}, CIFAR-10~\cite{krizhevsky2009learning}, and ImageNet-$32\times 32$~\cite{deng2009imagenet}. We use uniform dequantization and standard data augmentation techniques for CIFAR-10 and ImageNet-$32\times 32$ during training.

\textbf{Baselines.} We compare ButterflyFlow against several of the most relevant baselines in terms of methods and model architectures:  MAF~\cite{papamakarios2017masked}, Real NVP~\cite{dinh2016density}, Glow~\cite{kingma2018glow}, Emerging~\cite{hoogeboom2019emerging}, Woodbury~\cite{lu2020woodbury}, and i-ResNet~\cite{behrmann2019invertible}
We follow the standard experimental setups and architectural configurations as in prior works.

\looseness=-1
\textbf{Results.} Quantitative results are shown in Table~\ref{tab:density_estimation_image}, with visualizations of the generated samples in Figure~\ref{fig:gen}. 
We find that ButterflyFlow either outperforms or is on par with all relevant baselines. 
It achieves some improvements on CIFAR-10 and, on ImageNet-$32\times 32$ and MNIST, ButterflyFlow performs comparably to Glow, Emerging, and Woodbury (with the same Glow backbone). This is possibly due to overparametrization of these large models over image datasets; it is unlikely that adding more linear layers will yield significant improvements. We further examine this claim in Section~\ref{structured_exp} by evaluating the performance of shallower variants of ButterflyFlow on smaller datasets.
\begin{table}[ht]
    \centering
    \caption{Density estimation on image datasets. Test set log-likelihood values are in bits per dimension. Lower is better. ButterflyFlow performs favorably relative to all baselines.
    }
    \resizebox{\linewidth}{!}{
    \begin{tabular}{ccccccc}
    \Xhline{3\arrayrulewidth}
    &\multicolumn{1}{c}{MNIST}
    &\multicolumn{1}{c}{CIFAR-10}&\multicolumn{1}{c}{ImageNet 32$\times$32}\\
    \Xhline{2\arrayrulewidth}
     MAF~\cite{papamakarios2017masked} &1.89 &4.31 &-\\
     Real NVP~\cite{dinh2016density}  &1.06 &  3.49&  4.28\\
     Glow~\cite{kingma2018glow} &  1.05& 3.35 & 4.09 \\
     Emerging~\cite{hoogeboom2019emerging} & 1.05  & 3.34& 4.09\\
     Woodbury~\cite{lu2020woodbury}  & 1.05  & 3.35 & 4.09 \\
     Residual Flows~\cite{chen2019residual} &0.97 & 3.28 &4.01\\
     i-DenseNet~\cite{perugachi2021invertible} &- &3.25 &3.98\\
     i-ResNet~\cite{behrmann2019invertible} &1.06 & 3.45 & -\\
     ButterflyFlow (Ours) &{1.05} & {3.33} & {4.09} \\ 
    \Xhline{3\arrayrulewidth}
    \end{tabular}
    }
    \label{tab:density_estimation_image}
\end{table}

\begin{figure}
    \centering
     \begin{subfigure}[b]{0.32\linewidth}
         \centering
         \includegraphics[width=\textwidth]{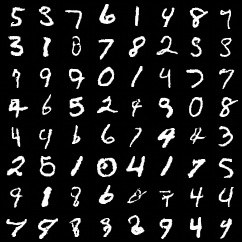}
         \caption{MNIST}
         \label{fig:mnist_viz}
     \end{subfigure}
     \begin{subfigure}[b]{0.32\linewidth}
         \centering
         \includegraphics[width=\textwidth]{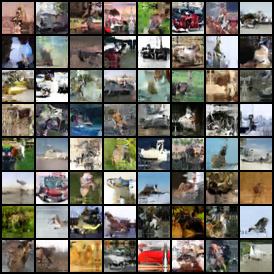}
          \caption{CIFAR-10}
         \label{fig:cifar_viz}
     \end{subfigure}
     \begin{subfigure}[b]{0.32\linewidth}
         \centering
         \includegraphics[width=\textwidth]{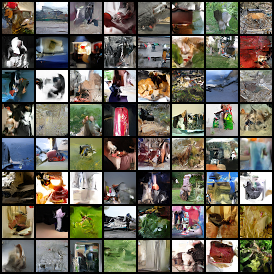}
          \caption{ImageNet-$32\times 32$}
         \label{fig:imagenet_viz}
     \end{subfigure}
     \caption{Uncurated samples from  ButterflyFlow. From left to right: MNIST, CIFAR-10, ImageNet-$32 \times 32$.}
     \label{fig:gen}
\end{figure}

\subsection{Density estimation on permuted image datasets}
\begin{figure*}
    \centering
         \begin{subfigure}[b]{0.19\linewidth}
         \centering
         \includegraphics[width=\textwidth,trim={0 91 0 0},clip]{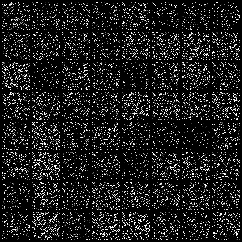}
         \caption{Permuted data (input)}
         \label{fig:data_viz}
     \end{subfigure}
     \begin{subfigure}[b]{0.19\linewidth}
         \centering
         \includegraphics[width=\textwidth,trim={0 91 0 0},clip]{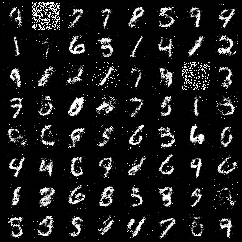}
         \caption{Glow}
         \label{fig:glow_perm_viz}
     \end{subfigure}
     \begin{subfigure}[b]{0.19\linewidth}
         \centering
         \includegraphics[width=\textwidth,trim={0 0 0 91},clip]{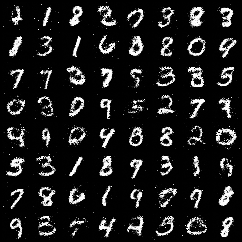}
         \caption{Emerging}
         \label{fig:emerging_perm_viz}
     \end{subfigure}
     \begin{subfigure}[b]{0.19\linewidth}
         \centering
         \includegraphics[width=\textwidth,trim={0 91 0 0},clip]{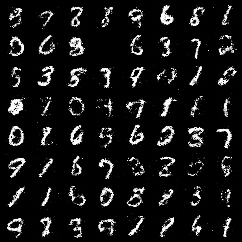}
         \caption{Woodbury}
         \label{fig:emerging_perm_viz}
     \end{subfigure}
     \begin{subfigure}[b]{0.19\linewidth}
         \centering
         \includegraphics[width=\textwidth,trim={0 0 0 91},clip]{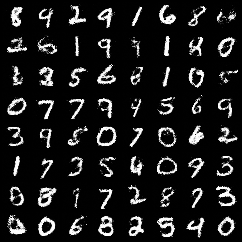}
         \caption{ButterflyFlow (Ours)}
         \label{fig:ours_perm_viz}
     \end{subfigure}
     \caption{Comparison of (unscrambled) generated samples on permuted MNIST. We observe that the Glow, Emerging, and Woodbury transforms struggle to model the permuted data well (as shown by missing, corrupted, or extremely speckled samples), while ButterflyFlow's learnable permutation layer allows it to better capture the permuted structure.
     } 
     \label{fig:permuted_gen}
\end{figure*}

In contrast to many existing linear transformations such as $1 \times 1$ convolutions, invertible butterfly layers are 
theoretically guaranteed to be able to represent a large family of complex linear transformations (\eg, any permutation matrix).
In this section, we demonstrate empirically that ButterflyFlow is expressive enough to capture special structures in the data such as permutations (as in \cref{pro:permutation}) by testing our model on image datasets with built-in permutations.

\textbf{Datasets.} We experiment on a permuted version of MNIST, CIFAR-10, and ImageNet-$32\times 32$ and generate a dataset-wide random permutation matrix. 
The same permutation matrix is used to permute all images from the same dataset. 

\textbf{Baselines.} We compare with Glow~\cite{kingma2018glow}, Emerging~\cite{hoogeboom2019emerging}, and Woodbury~\cite{lu2020woodbury}, which share the same architectural backbone as ButterflyFlow and similarly exploit spatial locality and permutation structures.

\looseness=-1
\textbf{Results.} We test the hypothesis that butterfly layers are more effective at capturing the structure in permuted images  as compared to baselines. 
Intuitively, this is because our butterfly layer is a learnable permutation layer that can capture permutation structure present in the data (Proposition \ref{pro:permutation}). 
The rest of the flow model can then learn the appropriate structure specific to the image dataset itself. 
As shown in Table~\ref{tab:density_estimation_image_perm}, we find that ButterflyFlow outperforms all other methods. 
Specifically, our method achieves significantly lower likelihoods as computed by bits per dimension (BPD) on CIFAR-10 and ImageNet-$32\times 32$.
The performance gap is noticeably closer for MNIST, and we show some visualizations of the generated images (permuted back) in Figure~\ref{fig:permuted_gen}. 
All our baselines are able to reasonably model permuted MNIST, likely due to the large modeling capacity of the Glow-based architecture on lower-dimensional datasets such as MNIST. Thus adding butterfly layers to specifically model permutation in this setting only yields marginal improvements.

\begin{table}[ht]
    \centering
    \caption{Density estimation on image  datasets with permutations. Test set log-likelihood values are reported in bits per dimension. Lower is better. ButterflyFlow outperforms all relevant baselines. %
    }
    \resizebox{\linewidth}{!}{
    \begin{tabular}{ccccccc}
    \Xhline{3\arrayrulewidth}
    &\multicolumn{1}{c}{MNIST}
    &\multicolumn{1}{c}{CIFAR-10}&\multicolumn{1}{c}{ImageNet 32$\times$32}\\
    \Xhline{2\arrayrulewidth}
     Glow~\cite{kingma2018glow} & 1.44 & 5.48 & 6.29 \\
     Emerging~\cite{hoogeboom2019emerging} & 1.43  &5.41  & 6.25\\
     Woodbury~\cite{lu2020woodbury} & 1.43  & 5.41 & 6.26\\
     ButterflyFlow (Ours) &\textbf{1.42} & \textbf{5.11} & \textbf{6.18} \\ 
    \Xhline{3\arrayrulewidth}
    \end{tabular}
    }
    \label{tab:density_estimation_image_perm}
\end{table}

\subsection{Density estimation on structured datasets}
\label{structured_exp}
Many real-world datasets often exhibit (unknown) special types of structures
such as permutation and periodicity. 
Therefore, in addition to modeling images with synthetic permutations, we also showcase a set of experiments where ButterflyFlow can be used to model real-world datasets with periodic structures. In particular, we experiment with galaxy images \citep{ackermann2018using,hoogeboom2019emerging} and the MIMIC-III patient records dataset~\cite{johnson2016mimic} of intensive care units (ICU). 

\paragraph{Galaxy images.} The galaxy dataset is comprised of 5000 images for both train and test sets, and exhibits periodicity as the images are ``continuous''---they represent snapshots of a continuum in space, rather than individual images. 
As shown in Table~\ref{tab:galaxy}, we find that ButterflyFlow outperforms all relevant baselines, achieving a BPD improvement of up to 0.07. We also visualize 100 generated images with 100 test set examples in Figure~\ref{fig:galaxy}.
This result provides further evidence that our invertible butterfly layers excel at capturing naturally-occuring structure in real-world data.
\begin{table}[ht]
\centering
\caption{Comparison of $1 \times 1$ convolutions (Glow), Emerging convolution, Woodbury flows, and ButterflyFlow on the galaxy images dataset. Test set log-likelihood values are reported in bits per dimension. Lower is better. ButterflyFlow outperforms all relevant baselines.}
    \resizebox{0.86\linewidth}{!}{
\begin{tabular}{ll}
\toprule
 & Galaxy \\
\midrule
$1 \times 1$ (Glow)~\citep{kingma2018glow} & 2.02\\ 
Emerging $3 \times 3$~\citep{hoogeboom2019emerging} & 1.98\\ 
Periodic~\citep{hoogeboom2019emerging} & 1.98\\ 
Woodbury~\citep{lu2020woodbury} & 2.01\\   
ButterflyFlow (Ours) & \textbf{1.95} \\ 
\bottomrule
\end{tabular}
}
\label{tab:galaxy}
\end{table}

\begin{figure}[ht]
\centering
     \begin{subfigure}[b]{0.3\linewidth}
         \centering
         \includegraphics[width=\textwidth]{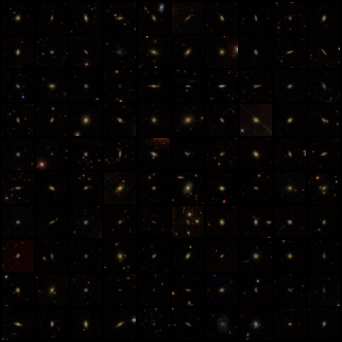}
     \end{subfigure}
     \begin{subfigure}[b]{0.3\linewidth}
         \centering
         \includegraphics[width=\textwidth]{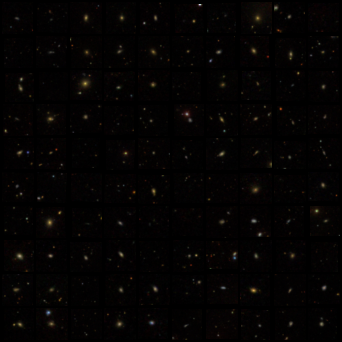}
     \end{subfigure}
\caption{(Left) 10 $\times$ 10 examples from the galaxy images test set. (Right) 10$\times$ 10 samples from the trained ButterflyFlow model. Note that the samples look visually similar.
}
\label{fig:galaxy}
\end{figure}

\paragraph{MIMIC-III waveform database.} 
MIMIC-III is a large-scale dataset containing approximately 30,000 patients' ICU waveforms. For each patient's waveform, two features are recorded: Photoplethysmography (PPG) and Ambulatory Blood Pressure (ABP). 
Since each patient's recording is very long,
we construct a per-patient dataset according to Appendix~\ref{app:mimic_details} and randomly select 3 distinct patient records for our experiments. 
We illustrate some example ground-truth waveforms in Figure~\ref{fig:mimic_data} and highlight its repetitive, periodic structure, which is difficult to capture faithfully with conventional flow-based generative models.

\begin{figure}[ht]
    \centering
     \begin{subfigure}[b]{0.49\linewidth}
         \centering
         \includegraphics[width=\textwidth]{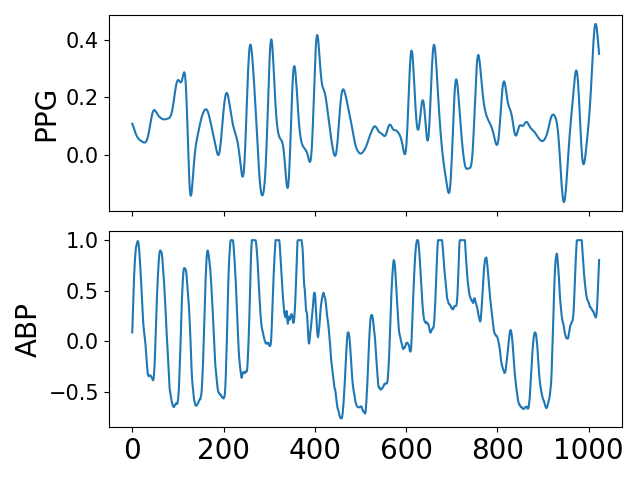}
     \end{subfigure}
     \begin{subfigure}[b]{0.49\linewidth}
         \centering
         \includegraphics[width=\textwidth]{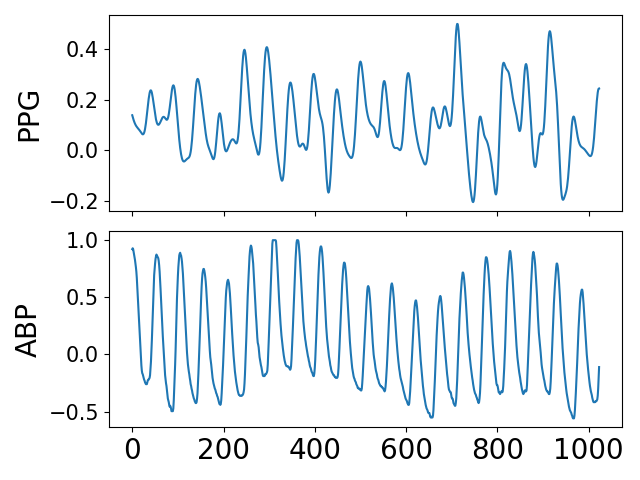}
     \end{subfigure}
     \caption{Two independent data points from the processed MIMIC-III patient waveforms. The x-axis indicates the 1024 intervals at which the signal was subsampled and the y-axis indicates the (normalized) recorded values for the PPG and ABP features.}%
     \label{fig:mimic_data}
\end{figure}

For modeling time series, we compare with Emerging and Periodic convolution baselines~\cite{hoogeboom2019emerging}, as well as Woodbury~\cite{lu2020woodbury}. 
All methods use the same Glow-based backbone of the same depth and levels.
As shown in Table~\ref{tab:density_estimation_mimic}, ButterflyFlow outperforms all baselines by a significant margin. In particular, our approach outperforms all competing methods while using \emph{less than half the number of parameters} required by the second-best performing model, as shown in Table~\ref{tab:density_estimation_mimic_params}. 
Thus, our model is more efficient in terms of space while better modeling the patient data with periodic regularity.

\begin{table}[ht]
    \centering
    \caption{Density estimation results on the MIMIC-III task. We report the test set negative log-likelihood per dimension. Lower is better. ButterflyFlow outperforms all other baselines by a significant margin.
    }
    \resizebox{\linewidth}{!}{
    \begin{tabular}{cccccccc}
    \Xhline{3\arrayrulewidth}
    &\multicolumn{1}{c}{Patient 1}
    &\multicolumn{1}{c}{Patient 2}&\multicolumn{1}{c}{Patient 3}&\multicolumn{1}{c}{Avg.}\\
    \Xhline{2\arrayrulewidth}
     Glow~\cite{kingma2018glow} &-7.21 &  -5.59 & -6.41 & -6.40\\
     Emerging~\cite{hoogeboom2019emerging} & -6.91 & -8.48  & -7.25 & -7.55 \\
     Periodic~\cite{hoogeboom2019emerging} & -8.47 &  -9.623 & -8.73 & -8.94 \\
     Woodbury~\cite{lu2020woodbury} & -11.68 &  -11.83 & -10.91 & -11.47 \\
     ButterflyFlow (Ours) & \textbf{-29.49} & \textbf{-27.07} & \textbf{-27.20} & \textbf{-27.92}\\
    \Xhline{3\arrayrulewidth}
    \end{tabular}
    }
    \label{tab:density_estimation_mimic}
    \vspace{-5mm}
\end{table}

\begin{table}[ht]
\centering
\caption{Total number of parameters for each model trained on the MIMIC-III dataset. ButterflyFlow is the best performing model (as in \cref{tab:density_estimation_mimic}) while using less than half the parameters as compared to baselines.
}
    \resizebox{0.8\linewidth}{!}{
\begin{tabular}{ll}
\toprule
 & \# parameters \\
\midrule
Glow~\citep{kingma2018glow} &36,032\\ 
Emerging~\citep{hoogeboom2019emerging} & 42,576\\ 
Periodic~\citep{hoogeboom2019emerging} & 39,312\\  
Woodbury~\citep{lu2020woodbury} & 48,576\\  
ButterflyFlow (Ours) & \textbf{15,280}  \\ 
\bottomrule
\end{tabular}
}
\label{tab:density_estimation_mimic_params}
\end{table}

Apart from natural image datasets, we find that our ButterflyFlow model shines when modeling real-world data with special underlying structures. Our empirical evaluations demonstrate that our invertible butterfly layers are able to better capture the global regularity than emerging or periodic convolutions, which rely on local spatial structures.

\looseness=-1
\subsection{Running time}
Finally, we benchmark the efficiency of ButterflyFlow,
which exploits the sparsity structure of its underlying butterfly factors.
We compare the forward and backward pass through a single butterfly layer with those of Emerging and Periodic convolution layers across 4 settings: forward/inversion time vs. spatial dimension size and forward/inversion time vs. batch size. We present additional details and comparisons in Appendix~\ref{app:runtime}.
\begin{figure}[ht]
    \centering
     
     \includegraphics[width=\linewidth]{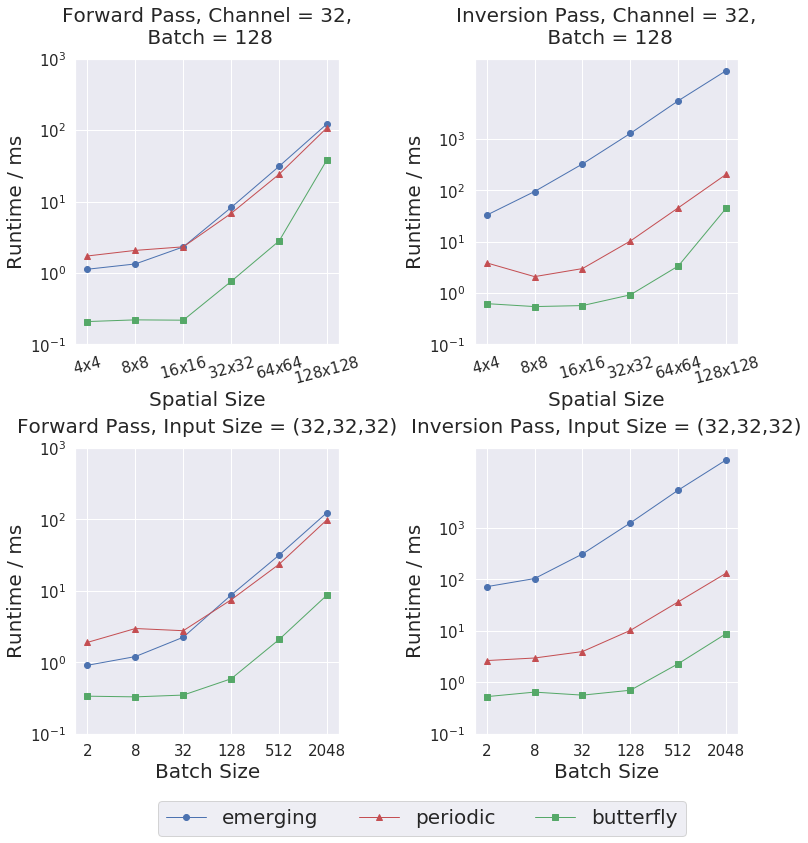}
     \caption{Run-time comparison. The y-axis shows run-time (\textbf{ms}) of each setting in log scale. Our run-time stays consistently lower.}%
     \label{fig:runtime}
     \vspace{-5mm}
\end{figure}
As shown in Figure~\ref{fig:runtime}, our runtime stays consistently lower than baselines, indicating that our butterfly layer is more computationally efficient.

\section{Conclusion}
In this work we proposed ButterflyFlow, a novel class of flow-based generative models parameterized by invertible butterfly layers. 
Drawing inspiration from the literature on learning efficient structured linear transformations, we introduced how butterfly layers more generally can serve as powerful architectural primitives for flow models.
We demonstrated that ButterflyFlow not only achieves strong performance on density estimation tasks for standard image datasets, but also better handles real-world data with naturally-occurring structures such as periodicity and permutations relative to existing baselines.
A current limitation of our approach is that we must manually specify a particular partitioning of the input for cases where its dimension is not divisible by 2. 
It would be interesting to generalize the invertible butterfly layer to handle such cases automatically.
Additionally, exploring further use cases of ButterflyFlow in applications beyond density estimation would be exciting future work. 

\section{Acknowledgement}
The authors would like to thank Jiaming Song for the constructive feedback. This research was support by NSF (\#1651565), AFOSR (FA95501910024), ARO (W911NF-21-1-0125) and Sloan Fellowship.

\nocite{langley00}

\bibliography{references}
\bibliographystyle{icml2022}

\newpage
\appendix
\onecolumn

\section{Proof}
\label{sec:proof}
In this section, we provide proofs for the main paper.
\begin{lemma}
\label{lemma:beuuterfly_level1}
The determinant of any (invertible) level-one butterfly factor $\textbf{B}_{\theta}{(1,D)}$ can be computed in $\mathcal{O}(D)$.
\end{lemma}
\begin{proof}
According to the definition of butterfly factor, we can write $\textbf{B}_{\theta}{(1,D)}=
\begin{bmatrix}
\textbf{D}_1 & \textbf{D}_2\\
\textbf{D}_3 & \textbf{D}_4
\end{bmatrix},$ 
where $\textbf{D}_i$ is a $D/2 \times D/2$ diagonal matrix. 
It is easy to see that 
$\text{det}(\textbf{B}_{\theta}{(1,D)})=\text{det}(\textbf{D}_1 \textbf{D}_4 - \textbf{D}_2 \textbf{D}_3)=\prod_{j=1}^{D/2}\big(\textbf{D}_1[j,j] \textbf{D}_4[j,j] -\textbf{D}_2[j,j] \textbf{D}_3[j,j]\big)$, where $\textbf{D}_i[j,j]$ denotes the $(j,j)$-th entry for $\textbf{D}_i$.
The Jacobian determinant of $\textbf{B}_{\theta}{(1,D)}$ can thus be computed in $\mathcal{O}(D)$.
\end{proof}

\determinant*

\begin{proof}
By definition, we have
\begin{equation*}
    \textbf{B}_{\theta}(i,D)=
    \begin{bmatrix}
    \textbf{B}_{1}(1, D/2^{i-1}), &\textbf{0} &... &\textbf{0}\\
    \textbf{0} &\textbf{B}_{2}(1, D/2^{i-1}) &... &\textbf{0}\\
    &...\\
    \textbf{0} &\textbf{0} &... &\textbf{B}_{2^{i-1}}(1, D/2^{i-1})
    \end{bmatrix}
\end{equation*}
where $\textbf{B}_{j}(1, D/2^{i-1})$ is a level-one $D/2^{i-1}$-dimensional (invertible) butterfly factor
and $\textbf{0}$ is the $D/2^{i-1}\times D/2^{i-1}$ zero matrix. 
Using the property of diagonal block matrices, we have
\begin{equation}
    \text{det}(\textbf{B}_{\theta}(i,D))=\prod_{j=1}^{2^{i-1}}\text{det}(\textbf{B}_{j}(1,D/2^{i-1})).
\end{equation}
From \cref{lemma:beuuterfly_level1}, we know computing each $\text{det}(\textbf{B}_{j}(1,D/2^{i-1}))$ takes $\mathcal{O}(D/2^{i-1})$, computing $\text{det}(\textbf{B}_{\theta}(i,D))$ thus takes $2^{i-1}\mathcal{O}(D/2^{i-1})=\mathcal{O}(D)$. 
\end{proof}

\inverse*
To prove \cref{pro:inverse},
we first prove \cref{lemma_butterfly_inverse}.

\begin{lemma} %
\label{lemma_butterfly_inverse}
Assuming $\textbf{B}_{\theta}{(1,D)}$ is non-singular, then its inverse  $\textbf{B}_{\theta}^{-1}{(1,D)}$ is a $D$-dimensional level-$i$ butterfly factor that can be computed in $\mathcal{O}(D)$ given $\textbf{B}_{\theta}{(1,D)}$.
\end{lemma}

\begin{proof}
According to the definition of butterfly factor, we can write $\textbf{B}_{\theta}{(1,D)}=
\begin{bmatrix}
\textbf{D}_1 & \textbf{D}_2\\
\textbf{D}_3 & \textbf{D}_4
\end{bmatrix},$ 
where $\textbf{D}_i$ is a $D/2 \times D/2$ diagonal matrix. 
The inverse of $\textbf{B}_{\theta}{(1,D)}$ can be computed as
\begin{equation}
\label{eq:butterfly_level1_inverse}
\textbf{B}^{-1}_{\theta}{(1,D)}=
\begin{bmatrix}
-\textbf{D}_4 / (\textbf{D}_3 \odot \textbf{D}_2 -\textbf{D}_4\odot \textbf{D}_1) & \textbf{D}_2 / (\textbf{D}_3\odot \textbf{D}_2 - \textbf{D}_4\odot \textbf{D}_1)\\
-\textbf{D}_3 / (\textbf{D}_1 \odot \textbf{D}_4 - \textbf{D}_3\odot \textbf{D}_2) & \textbf{D}_1 / (\textbf{D}_1\odot \textbf{D}_4 - \textbf{D}_3\odot \textbf{D}_2) 
\end{bmatrix},
\end{equation}
where $\textbf{D}_i\odot \textbf{D}_j$ are element-wise multiplication of diagonal matrices. Since $\textbf{D}_i$ is a $D/2\times D/2$ diagonal matrix, computing $\textbf{D}_i\odot \textbf{D}_j$ can be performed in $\mathcal{O}(D/2)$. Thus, evaluating $\textbf{B}^{-1}_{\theta}{(1,D)}$ can be performed in $\mathcal{O}(D)$.

Since each $\textbf{D}_i$ is a $D/2\times D/2$ diagonal matrix, each of the block in \cref{eq:butterfly_level1_inverse} are also diagonal. Thus, $\textbf{B}^{-1}_{\theta}{(1,D)}$ is a level-one $D$-dimensional butterfly block by definition.
\end{proof}

We now prove \cref{pro:inverse}.
\begin{proof}[Proof of \cref{pro:inverse}]
According to the definition of $\textbf{B}_{\theta}{(i,D)}$, we can write it as
\begin{equation*}
    \textbf{B}_{\theta}(i,D)=
    \begin{bmatrix}
    \textbf{B}_{1}(1, D/2^{i-1}), &\textbf{0} &... &\textbf{0}\\
    \textbf{0} &\textbf{B}_{2}(1, D/2^{i-1}) &... &\textbf{0}\\
    &...\\
    \textbf{0} &\textbf{0} &... &\textbf{B}_{2^{i-1}}(1, D/2^{i-1})
    \end{bmatrix}
\end{equation*}
where $\textbf{B}_{j}(1, D/2^{i-1})$ is a level-one $D/2^{i-1}$-dimensional (invertible) butterfly factor
and $\textbf{0}$ is the $D/2^{i-1}\times D/2^{i-1}$ zero matrix. 
Using the properties of diagonal block matrices, it is easy to check
\begin{equation*}
    \textbf{B}^{-1}_{\theta}(i,D)=
    \begin{bmatrix}
    \textbf{B}^{-1}_{1}(1, D/2^{i-1}), &\textbf{0} &... &\textbf{0}\\
    \textbf{0} &\textbf{B}^{-1}_{2}(1, D/2^{i-1}) &... &\textbf{0}\\
    &...\\
    \textbf{0} &\textbf{0} &... &\textbf{B}^{-1}_{2^{i-1}}(1, D/2^{i-1})
    \end{bmatrix}.
\end{equation*}
According to \cref{lemma_butterfly_inverse}, we have each $\textbf{B}^{-1}_{j}(1, D/2^{i-1})$ is a level-one $D/2^{i-1}$-dimensional butterfly factor that can be computed in $\mathcal{O}(D/2^{i-1})$ given $\textbf{B}_{j}(1, D/2^{i-1})$. Thus, $\textbf{B}^{-1}_{\theta}(i,D)$ is a level-$i$ $D$-dimensional butterfly factor by definition. It can be computed in $2^{i-1}\mathcal{O}(D/2^{i-1})=\mathcal{O}(D)$ given $\textbf{B}_{\theta}(i,D)$.
Since $\textbf{B}^{-1}_{\theta}(i,D)$ is a $D\times D$ sparse matrix with only two non-zero entries each row,
 the map $b_i^{-1}: \mbf{z}\to \textbf{B}_{\theta}^{-1}{(i,D)}\mbf{z}$ (\ie, a matrix vector multiplication) can be computed in $2\mathcal{O}(D)=\mathcal{O}(D)$ given $\textbf{B}_{\theta}^{-1}{(i,D)}$.
\end{proof}

\permutation*
\begin{proof}
According to Theorem 2. in \cite{dao2020kaleidoscope}, any $D\times D$ permutation matrix $\textbf{P}\in\mathbb{R}^{D\times D}$ (when $D=2^k$) can be represented as 
\begin{equation}
    \textbf{P}= b_1\circ b_2\circ... \circ b_{k-1} \circ b_k \circ \hat{b}_k \circ \hat{b}_{k-1}\circ... \circ \hat{b}_1(\textbf{I}),
\end{equation}
where $b_i:\mbf{x}\to \textbf{B}_{\theta}(i, D)\mbf{x}$ and $\hat{b}_i:\mbf{x}\to \textbf{B}_{\hat\theta}(i, D)\mbf{x}$ are linear layers obtained by multiplying a learnable level-$i$ $D$-dimensional butterfly matrix with the input.
Since $\textbf{P}$ is a permutation matrix, it is non-singular, which implies that each $b_i$, $\bar b_i$ and $\hat{b}_i$ must be invertible. Thus, any $D\times D$ permutation matrix can be represented by an \emph{invertible} butterfly layer.
We also note that \cite{dao2020kaleidoscope}  does not consider settings where
exponentiation of a linear transformation is also invertible (as in our invertible butterfly layers).

\end{proof}

\begin{lemma}[\cite{dao2020kaleidoscope}]
\label{lemma:conv}
Any $D\times D$ ($D=2^{k}$) convolution matrix $\mathbf{C}_{D}$ can be represented as 
\begin{equation}
    \mathbf{C}_{D} = b_{1} \circ b_{2} \circ ... \circ b_{k} \circ \hat{b}_{k-1}  \circ \hat{b}_{k-2}  \circ ... \circ \hat{b}_{1}, 
\end{equation}
where $b_i:\mbf{x}\to \textbf{B}_{\theta}(i, D)\mbf{x}$ and $\hat{b}_i:\mbf{x}\to \textbf{B}_{\hat\theta}(i, D)\mbf{x}$
are butterfly layers with weights in $\mathbb{C}$.
\end{lemma}
\begin{proof}
See Lemma J.5. in \cite{dao2020kaleidoscope}.
\end{proof}

\begin{proposition} %
Given a single channel 2D input  $\mbf{x}\in\mathbb{R}^{W\times H}$, any 2D convolution layer with kernel size $k\times k$, zero padding and output channel one, can be obtained by multiplying a circulant matrix with the input with padding expanded to a vector.
\end{proposition}
\begin{proof}[Sketch of proof]
Given the input $\mbf{x}$, we apply the zero padding to $\textbf{x}$ and obtain a padded input $\tilde{\mbf{x}}$. We then expand $\tilde{\mbf{x}}$ to a one-dimensional vector. It is easy to show that the 2D convolution can be represented as a circulant matrix multiplied by $\tilde{\mbf{x}}$ with entries (of the output) corresponding to the paddings removed.
\end{proof}

\observationconv*
\begin{proof}[Sketch of proof]
Given an input $\mbf{x}\in \mathbb{R}^{C, W, H}$, an invertible $d\times d$ convolution can be decomposed into two steps: (1) mix the channel information
for each $(w,h)$ pair, $w\in[W]$ and $h\in [H]$,
by performing an invertible $C\times C$ matrix multiplication with a $C$-dimensional vector  $\mbf{x}[:, w,h]$, and (2)
perform single channel $d\times d$ convolution for each of the $C$ inputs $\mbf{x}[i, :, :]$, $i=1,...,C$, independently. As we showed previously, each of the single channel $d\times d$ convolution can be performed by using circulant matrix, vector multiplication. For input whose size after padding is not a power of $2$, we can always pad extra zeros so that the input after padding has size of power of $2$. We can remove the entries corresponding to the paddings in the output to recover the correct output. Now, observe that each $C\times C$ matrix block in the block-wise butterfly matrix exactly corresponds to (1) and by \cref{lemma:conv}, any circulant matrix with size a power of $2$ can be represented using naive butterfly layers. Then the $D/C\times D/C$ block matrix in block-wise butterfly factors (seeing each $C\times C$ as a whole) corresponds to (2).
Thus block-wise invertible butterfly layer with weights in $\mathbb{C}$ can be used to represent a family of the invertible $d\times d$ convolution layer.
\end{proof}

\section{Experiments}
\label{app:exp_details}

\subsection{Training details}

For all experiments, we use Adam optimizer with $\alpha = 0.001$, $\beta_1 = 0.9$, $\beta_2 = 0.999$ for training. We warm up our learning by linearly increasing learning rate from 0 to initial learning rate for 10 iterations, and afterwards exponentially decaying with $\gamma = 0.999997$ per iteration. Training is done on TITAN RTX GPU machines. For some experiments we also employ exponential moving average (EMA) of either the entire model or only the butterfly layers, which we will specify in the next section. 

\subsection{Model architecture}
\label{app:model}

\begin{table}[t]
\centering
\caption{Model architecture for various datasets.}
\resizebox{\textwidth}{!}{
\begin{tabular}{ccccccccc}
\toprule
 & Levels (L) & Steps (K) & Coupling channels & Butterfly levels & Bi-direction? & EMA & Butterfly scheduler $\gamma$ & Butterfly init\\
\midrule
CIFAR-10 & 3 & 32 & 512 & 1 & \xmark & none & N/A & id\\ 
ImageNet-$32\times 32$ & 3 & 32 & 512 & 1 &\xmark & none& N/A & id\\ 
MNIST &  2 & 20 & 512 & 1&\xmark & none& N/A & id\\ 
CIFAR-10,permuted                & 3 & 32 & 512 & 10 & \xmark & separate & 0.99 & rot\\ 
ImageNet-$32\times 32$,permuted   & 3 & 32 & 512 & 10 & \cmark & separate & 0.99 & rot\\ 
MNIST,permuted                    & 2 & 20 & 512 & [9,8,4] &\cmark & separate&0.996 & id\\ 
Galaxy  & 2 & 8 & 512 & 2 & \xmark & separate & 0.996 & id \\
MIMIC-III  & 2 & 2 & 16 & 10 & \xmark & none & N/A & rot \\
\bottomrule
\end{tabular}
}
\label{tab:architecture}
\end{table}

We here define relevant model architecture hyperparameters. The backbone of the our network follows Glow~\cite{kingma2018glow} baseline as visualized in Figure~\ref{fig:butterfly_flow}. Our model uses $L$ levels and $K$ steps, and each butterfly layer is of maximum $M$ levels. We by default choose a list of contiguous integers to parametrize our levels $\{a_i\}_{i=1}^k$, \ie, for a butterfly layer of $M$ levels, $\{a_i\}_{i=1}^k = \{1,2,\dots,M\}$. For our butterfly layers we also implement a version specified in Proposition~\ref{pro:permutation}, which stacks a level-inverted $M$-level butterfly layer on top of a regular butterfly layer. We indicate this version as ``bi-direction" in Table~\ref{tab:architecture}. If it is set, our butterfly layer has $2M$ butterfly factors with selected integer set $\{a_i\}_{i=1}^k = \{1,2,\dots,M,\dots,2,1\}$. For our models, we also implement different types of parameter EMA for training. When EMA is ``none", we use a single Adam optimizer for all parameters. When EMA is indicated as ``all", we employ EMA on all model parameters. When EMA is indicated as ``separate", we employ EMA only for all of our butterfly layers. During training, we use a separate Adam optimizer of the same hyperparameters and exponential decay scheduler of different $\gamma$ for butterfly layers than the Glow backbone, and we optimize the Glow backbone based on the EMA output of butterfly layers. 

We also explore different initialization types for our butterfly layers. If it is ``id", we initialize all our butterfly factors to identity matrix. If it is ``rot", we initialize our butterfly factors such that the 4 diagonal matrices are element-wise orthogonal. That is, if a butterfly factor is $\begin{bmatrix}
    \textbf{D}_1 &\textbf{D}_2 \\
    \textbf{D}_3 &\textbf{D}_4\\
    \end{bmatrix}$ with each sub-matrix being a diagonal matrix, each $2\times 2$ matrix $\begin{bmatrix}\textbf{D}_1[k,k] & \textbf{D}_2[k,k]\\\textbf{D}_3[k,k] & \textbf{D}_4[k,k]\\ \end{bmatrix}$ is initialized to a rotation matrix.

\textbf{Image datasets.} For MNIST datasets specifically, we use logit transform with $\lambda=10^{-6}$ for data preprocessing. For CIFAR-10 and ImageNet-$32\times 32$, we follow~\cite{hoogeboom2019emerging} for data preprocessing.

\textbf{Permuted image datasets.} For ImageNet-$32\times 32$ and CIFAR-10 in particular, we use level-10 butterfly layers and decrease the level by 1 after each Squeeze layer. Since MNIST's image size is $28\times 28 = 784$, it is not divisible by 2 as required by butterfly layers. 
Therefore, we choose to partition the space into a concatenation of $512,256,16$-dimensional spaces where each can be fed into a $9,8,4$-level butterfly factor respectively. Each separate butterfly matrix's level decreases by 1 after each Squeeze layer.

\textbf{Galaxy images.} Model architecture is as shown in Table~\ref{tab:architecture} and we empirically find the using batch size 64 results in better performance.

\textbf{MIMIC-III waveform database.} Since the data has shape (1024, 2), we treat each data point as a 1D image of size 1024 and 2 channels. We then straight-forwardly adapt the Glow backbone for 2D image to process 1D data. For our Emerging and Periodic baselines, we use filter size of 51 since we empirically found that using the default value filter size of 3 fails in learning a reasonable density estimator. For all our model we also use learning rate of 0.0001 because we observed that higher learning rate results in unstable loss curves. For our butterfly matrix, to reduce the number of learnable parameters, we also share parameters for each diagonal matrix, \ie, if a butterfly factor is $\begin{bmatrix}
    \textbf{D}_1 &\textbf{D}_2 \\
    \textbf{D}_3 &\textbf{D}_4\\
    \end{bmatrix}$ with each sub-matrix being a diagonal matrix, the diagonal elements in each $\textbf{D}_k$ are the same. This sharing holds for each primitive diagonal matrix in a butterfly factor. 
    
\textbf{Running time.} 
\label{app:runtime}
We define specifically what is forward pass and inversion pass for each layer tested. In PyTorch's language, by forward pass we mean applying the tested layer \emph{and} computing the log determinant of its Jacobian under \texttt{requires\_grad} mode. By inversion pass we mean applying the inverse of the tested layer under \texttt{no\_grad} mode. 

All testing is done on a TITAN XP GPU. For each tensor tested, e.g. of size $3\times 32\times 32$, we flatten it into a vector before applying butterfly matrix in our CUDA implementation. We use level-10 butterfly layer by default, and for tensors of smaller sizes, we use the maximum possible level to construct our butterfly layer. For example, a tensor of size $1\times 4\times 4$ allows for a level-4 butterfly layer. For tensors of large sizes, e.g. $3\times 128\times 128$, which allows for butterfly layers with more than 10 levels, we stop at level 10 because it is the maximum number we use in all our other experiments. 

Here we also present additional comparisons with $1\times 1$ convolution. 

\begin{figure}[h]
    \centering
    \includegraphics[width=0.65\linewidth]{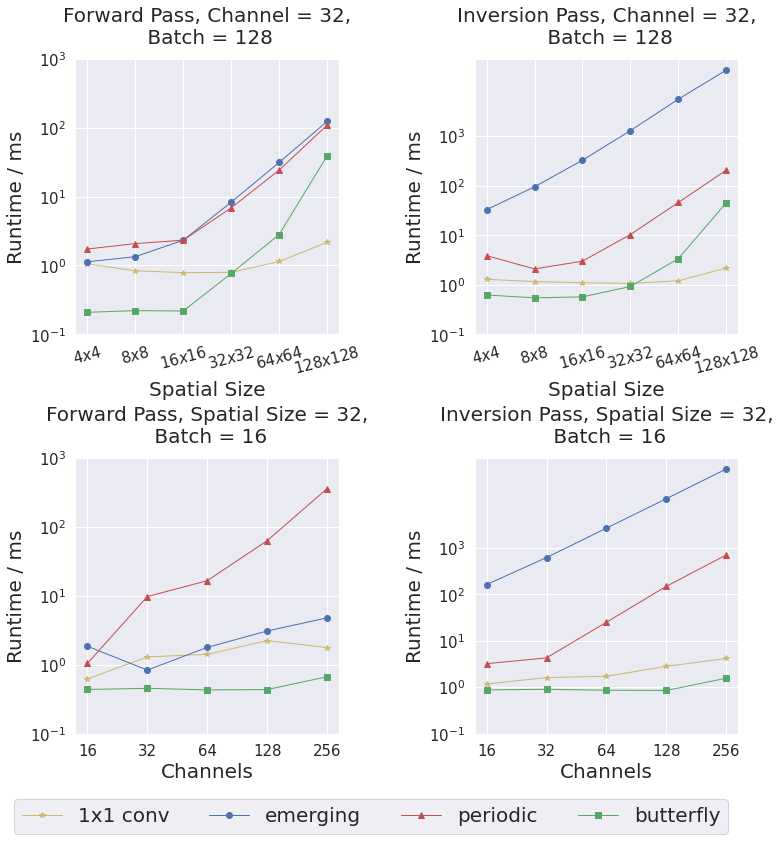}
    \caption{Comparisons with $1\times 1$ convolution added.}
    \label{fig:additional_runtime}
\end{figure}

Shown in \cref{fig:additional_runtime}, $1\times 1$ convolution scales better than our model on large images because the operation can be parallelized across all spatial locations.  We note that our model is faster on smaller images (first row of \cref{fig:additional_runtime}) and performs comparably with $1\times 1$ convolution at spatial size $32\times 32$. Nevertheless, $1\times 1$ convolution does not scale well with increasing channel size primarily because calculation of its determinant is cubic with respect to channel size. For fair comparison, with fix spatial size at $32\times 32$ and vary channel size (second row of \cref{fig:additional_runtime}). We find that our model outperforms all baselines for runtime vs. channel size. 

\section{Datasets}
\label{app:dataset}
\subsection{Permuted image datasets}

For each of CIFAR-10, ImageNet-$32\times 32$, and MNIST, we generate a random permutation matrix and preprocess the images in each dataset using the same dataset-wise permutation matrix. visualizations are done by first generating from the model and permute back using the ground-truth matrix.

\subsection{MIMIC-III waveform database}
\label{app:mimic_details}
MIMIC-III is a large-scale dataset containing approximately 30,000 patients' ICU waveforms. Each patient's record contains a time series of periodic measurements, which is a quasi-continuous recording of the patient's vital signals over their entire stay at the hospital (sometimes days and usually weeks). For this dataset in particular, two feature waveforms are recorded by bedside monitors: Photoplethysmography (PPG) and Ambulatory Blood Pressure (ABP) waveforms. 

Due to the extremely long samples per patient, we built per-patient datasets by cutting each waveform sequence into chunks of length 1024. 
As a concrete example, we can build a dataset of 10,000 data points for a patient with 10.24M sampled intervals. Within this patient's recording, we then have 10,000 data points of dimension (1024, 2) where each dimension corresponds to PPG and ABP features in time. The data points are additionally normalized to $[-1, 1]$ before training. Patient 1, 2, 3 corresponds to patient ID 3000063, 3000393, 3000397, respectively. More details about the dataset is available at \url{https://physionet.org/content/mimic3wdb/1.0/}. We also preprocess our data according to~\cite{s19153420} with this Github page \url{https://github.com/gslapnicar/bp-estimation-mimic3}, which performs necessary filtering for noise removal and anomaly removal.

\end{document}